\documentclass[sigconf,authorversion, nonacm, natbib=true]{acmart}
\usepackage{booktabs} 
\usepackage{graphicx}
\usepackage{balance}  
\usepackage[ruled,vlined,linesnumbered]{algorithm2e}
\usepackage{algorithmic}
\usepackage{amsthm}
\usepackage{adjustbox}
\usepackage{xspace}
\usepackage{hyperref}
\usepackage{amsmath}
\usepackage{url}
\usepackage{autobreak}
\usepackage{colortbl}
\usepackage{color}
\usepackage{fancyvrb,xcolor}
\usepackage[center]{subfigure} 
\usepackage[mathscr]{euscript} 
\usepackage{makecell}
\usepackage{footnote} 
\usepackage{tablefootnote} 
\usepackage{mdwlist}
\usepackage{multirow}
\usepackage{enumitem}
\usepackage{lipsum}

\newcolumntype{s}{>{\columncolor{gray!30}} r}

\AtBeginDocument{%
  }

\setcopyright{acmlicensed}
\copyrightyear{2025}
\acmYear{2025}
\acmDOI{XXXXXXX.XXXXXXX}

\acmConference[SIGIR'25]{The 48th International ACM SIGIR Conference on Research and Development in Information Retrieval}{July, 2025}{Padova, Italy}
\acmISBN{978-1-4503-XXXX-X/18/06}

\newtheorem{problem}{Problem}

\newtheorem{definition}{Definition}

\SetKwComment{Comment}{$\triangleright$\ }{}
\theoremstyle{remark}
\newtheorem*{remark}{Remark}
\definecolor{beige}{HTML}{FFFFCC}

\newcommand{\method}{\textsc{TCN}\xspace}

\newcommand{\costco}{\textsc{CostCo}\xspace}

\newcommand{\footnoteref}[1]{\textsuperscript{\ref{#1}}}

\newcommand{\hide}[1]{}

\newcommand{\mat}[1]{\mathbf{#1}}

\newcommand{\T}[1]{\boldsymbol{\mathcal{#1}}}

\newcommand\red[1]{\textcolor{red}{#1}}
\newcommand\blue[1]{\textcolor{blue}{#1}}

\begin{document}

\title{
Tensor Convolutional Network for Higher-Order Interaction Prediction in Sparse Tensors
}

\author{Jun-Gi Jang}
\email{jungi@illinois.edu}
\affiliation{%
  \institution{University of Illinois}
  \city{Urbana-Champaign}  
  \country{USA}
}
\author{Jingrui He}
\email{jingrui@illinois.edu}
\affiliation{%
  \institution{University of Illinois}
  \city{Urbana-Champaign}  
  \country{USA}
}
\author{Andrew Margenot}
\email{margenot@illinois.edu}
\affiliation{%
  \institution{University of Illinois}
  \city{Urbana-Champaign}  
  \country{USA}
}
\author{Hanghang Tong}
\email{htong@illinois.edu}
\affiliation{%
  \institution{University of Illinois}
  \city{Urbana-Champaign}  
  \country{USA}
}

\renewcommand{\shortauthors}{Jang et al.}

\begin{abstract}
 Many real-world data, such as recommendation data and temporal graphs, can be represented as incomplete sparse tensors where most entries are unobserved.
For such sparse tensors, identifying the top-{\em k} higher-order interactions that are most likely to occur among unobserved ones is crucial.
Tensor factorization (TF) has gained significant attention in various tensor-based applications, serving as an effective method for finding these top-{\em k} potential interactions.
However, existing TF methods primarily focus on effectively fusing latent vectors of entities, which limits their expressiveness.
Since most entities in sparse tensors have only a few interactions, their latent representations are often insufficiently trained.
In this paper, we propose \method, an accurate and compatible tensor convolutional network that integrates seamlessly with existing TF methods for predicting higher-order interactions.
We design a highly effective encoder to generate expressive latent vectors of entities.
To achieve this, we propose to (1) construct a graph structure derived from a sparse tensor and (2) develop a relation-aware encoder, \method, that learns latent representations of entities by leveraging the graph structure.
Since \method complements traditional TF methods,
we seamlessly integrate \method with existing TF methods, enhancing the performance of predicting top-$k$ interactions.
Extensive experiments show that \method integrated with a TF method outperforms competitors, including TF methods and a hyperedge prediction method.
Moreover, \method is broadly compatible with various TF methods and GNNs (Graph Neural Networks), making it a versatile solution.
\end{abstract}



\maketitle

\section{Introduction}
\label{sec:intro}
Tensor, a multi-dimensional array, is a natural representation of real-world datasets that contain higher-order relationships.
An $N$-th order tensor of the size $I_1\times \cdots \times I_N = T$ has $T$ entries each of which has a value indexed by indices $(i_1,...i_n,...,i_N)$ where $i_n$ denotes $i_n$-th entity in $n$-th dimension.
Specifically, many real-world data such as recommendation data, knowledge graphs, social networks, and web logs can be represented as sparse tensors whose most entries are unobserved.
In other words, there are only a few observed interactions between entities.
For example, given higher-order recommendation data where entities of each dimension correspond to users, location, and timestamps, respectively, an $(i_1, i_2, i_3)$-th entry value indicates an implicit (or explicit) feedback that the $i_1$-th user is interested in the $i_2$-th location at the $i_3$-th timestamp.
Most entries are unobserved since most users provide feedback for only a few items at specific timestamps:
as summarized in Table~\ref{tab:Description}, 99.9984\% entries are unobserved on SG dataset~\cite{li2015rank} which is POI (Point-of-Interest) recommendation data.

Given an $N$-th order sparse tensor, it is crucial to infer potential interactions between $N$ entities $i_1,...,i_N$ on different dimensions.
We need to identify top-{\em k} interactions carefully among unobserved interactions since it is not the desired situation that all interactions in a sparse tensor are observed.
For example, given a 3rd-order tensor of recommendation data (user, item, time), 
it is not realistic for every user to interact with every item at every time slot.
Each user has distinct preferences only for a few specific items, and there are dedicated browsing time slots.
Therefore, observed interactions in a sparse tensor represent meaningful higher-order relationships between entities (e.g., users, items, and times).
Identifying top-{\em k} interactions has various real-world applications such as recommendation systems, temporal graph analysis, and knowledge graph completion.
In addition, one such application is an enhanced tensor completion where only significant entries among all unobserved entries are filled carefully; (1) we identify top-k interactions using an identification method and then (2) fill in the values for the top-k interactions using a completion method.

Then, how can we find top-{\em k} potential interactions between entities accurately?
For predicting top-$k$ potential interactions, we can use tensor factorization methods such as traditional methods (e.g., CP~\cite{carroll1970analysis,harshman1970foundations,kiers2000towards} and Tucker~\cite{tucker1966some} decomposition) and deep learning-based methods (e.g., \costco~\cite{liu2019costco} and NeAT~\cite{ahn2024neat}).
Tensor factorization has garnered significant attention across various applications such as tensor completion~\cite{tomasi2005parafac,acar2011scalable,smith2016exploration,liu2014generalized}, anomaly detection~\cite{JangK21,jang2023kdd}, recommendation~\cite{choi2019s3,chen2020neural}, knowledge graph completion~\cite{BalazevicAH19,LacroixOU20}, and so on~\cite{PerrosPWVSTS17,OhPSK18,chen2018drugcom}.
A commonly considered approach is to (1) design a score function using a tensor factorization method, (2) train its parameters to make observed interactions high scores, (3) compute scores of unobserved interactions using the score function, and (4) pick top-$k$ highest-scored ones.
However, traditional TF methods face a significant challenge for incomplete sparse tensors where the latent vectors of most entities with a limited number of interactions remain undertrained due to high sparsity.
This is because traditional TF methods typically focus on how to fuse multiple latent vectors effectively.

In this paper, we propose \method, an accurate \textbf{\underline{T}}ensor \textbf{\underline{C}}onvolutional \textbf{\underline{N}}etwork (TCN) that is fully compatible with a variety of tensor factorization methods for higher-order interaction prediction.
To address the challenge of traditional TF methods, our proposition is to (1) extract a graph structure from a given sparse tensor and (2) generate enriched latent vectors of entities by aggregating information over the graph.
Specifically, we construct a hypergraph structure from a given $N$-th order sparse tensor where its entities and interactions correspond to nodes and hyperedges of the size $N$, respectively.
Then we obtain a clique expansion graph from the hypergraph where each interaction (i.e., hyperedge) of an $N$-th order tensor can be viewed as a $N$-clique.
Next, we design \method that extracts relation-aware latent vectors of entities by neighborhood aggregation through the graph.
\method effectively addresses the issue of poorly learning the latent vectors for entities with limited interactions.
Since the output of \method can be readily utilized as the input of a tensor factorization method,
we construct an end-to-end model that consists of \method and a tensor factorization method.
\method is highly compatible with a variety of TF methods since \method and TF methods serve different but complementary roles.

The main contributions of this paper are summarized as follows:
\begin{itemize}[noitemsep,topsep=0pt]
	\item \textbf{A New Insight.} We provide a novel insight in a tensor-based problem: (1) a construction of hypergraph and its clique-expanded graph derived from a sparse tensor, and (2) a development for a relation-aware encoder that learns latent vectors of entities using the graph.
	\item \textbf{Accurate Method.} We propose \method, which is highly accurate in achieving the state-of-the-art method for top-$k$ higher-order interaction prediction.
	\item \textbf{Plug-and-Play.} The plug-in capability for various tensor factorization (See Table~\ref{tab:effectiveness}) is the strength of \method.
		In addition, our tensor-originated clique expansion graph enables tensor factorization models to employ various GNNs to boost their performance (See Table~\ref{tab:ablation}).
	\item \textbf{Evaluation.} We experimentally show \method consistently boosts the performance of various existing methods. Furthermore, \method with \costco significantly outperforms competitors on sparse tensors in terms of average precision.
\end{itemize}

\vspace{-2mm}

\section{Preliminaries}
\label{sec:prelim}
%

$\T{X}$ denotes a tensor, and an index $i_n$ indicates the $i_n$-th entity of the $n$-th dimension.
$(i_1,...,i_N)$ is an interaction between $N$ entities $i_1,...,i_N$, and $I_n$ is the size of the $n$-th dimension.
A set $\Omega_o$ of observed interactions and a set $\Omega_u$ of unobserved interactions are mutually disjoint sets whose union is $\Omega_a$ which is a set of all interactions in a sparse tensor.

\begin{figure*}[t]
	\centering	
	\vspace{-3mm}
	\includegraphics[width=0.75\textwidth]{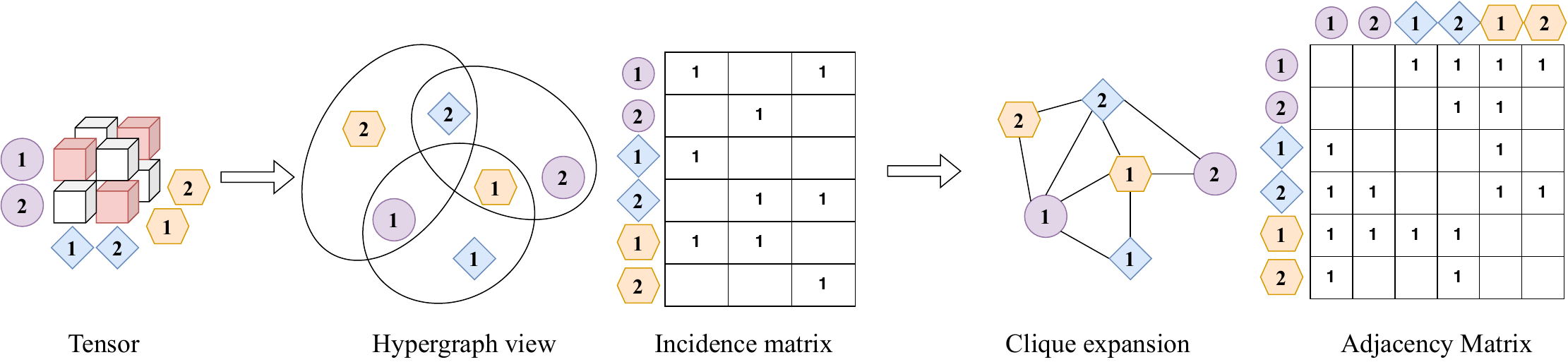}
	\caption{
	A hypergraph and its clique-expanded graph generated from a sparse tensor.
	See the detail in Section~\ref{subsec:hypergraph}.
 }
	\label{fig:graph_construction}
\end{figure*}

\subsection{Tensor Factorization}
\label{subsec:tensor_factorization}

Tensor factorization has been widely used for analyzing real-world tensors.
In the training phase, it learns decomposition results which minimizes a loss function (e.g., mean squared error (MSE) loss).
Refer to~\cite{liu2019costco}, tensor factorization learns (1) factor matrices $\mat{A}^{(1)},...,$ $\mat{A}^{(N)}$ corresponding to each dimension and (2) parameters $\theta$.
\begin{align}
\label{eq:tf_elementwise}
	\hat{\T{X}}(i_1,...,i_N) \leftarrow f(i_1,...,i_N;\mat{A}^{(1)},...,\mat{A}^{(N)}, \theta)
\end{align}
In the inference phase, given an interaction $(i_1,...,i_N)$, it predicts the value of $(i_1,...,i_N)$ by using the factor matrices and the parameters as shown in Equation~\eqref{eq:tf_elementwise}.

Now, let us look into a function $f$ of tensor factorization.
In Equation~\eqref{eq:tf_elementwise}, we can view tensor factorization as consisting of two modules:
(1) a one-hot encoder module that loads latent vectors $\mat{A}^{(n)}(i_n,:)^T$ for $n=1,...,N$ of entities $i_1$,...,$i_N$ from factor matrices and
(2) a predictor module that predicts a score using the latent vectors of the entities.
$\mat{A}^{(n)}(i_n,:)$ is a latent vector of the $i$-th entity of the $n$-th dimension.
Existing tensor factorization methods focus primarily on designing an effective predictor with latent vectors.
For example, given an interaction $(i_1,...,i_N)$, \costco utilizes a one-hot encoder that gets latent vectors of entities $i_1,...,i_N$ to construct the vertical concatenation $\vee_{n=1}^{N}(\mat{A}^{(n)}(i_n,:))\in\mathbb{R}^{N\times r}$.
Note that $\vee$ denotes the vertical concatenation of vectors.
Then, it utilizes CNN (Convolutional Neural Network) and MLP (Multi-Layer Perceptron) that capture complex relationships in the concatenation, for accurate prediction.

\subsection{Higher Order Interaction Prediction (HOIP)}
\label{subsec:hoir}

Before presenting the proposed method, we introduce top-$k$ higher-order interaction prediction.
We give the problem definition of higher-order interaction prediction addressed in this paper:
\begin{problem}\label{prob:main}
\textsc{\normalfont{(Top-$k$ higher-order interaction prediction)}} 
\textbf{Given} (1) a sparse tensor $\T{X} \in \mathbb{R}^{I_1 \times \cdots \times I_N}$, (2) a value $k$, \textbf{find} the top-$k$ interactions with the highest scores in the set $\Omega_u$.
\end{problem}

In this paper, we design a score function $f$ using tensor factorization, learn its parameters from a set $\Omega_o$ containing observed interactions, and identify the top-$k$ interactions based on this score function.

\textbf{Loss function.}
To identify potential higher-order interactions effectively, we use a pairwise loss function.
This loss function makes a model predict high scores for positive interactions while predicting low scores for negative interactions.
\begin{align}
	\label{eq:loss_function}
	\T{L} = \sum_{(e^{+}, e^{-})} {-\log \sigma(\hat{\T{X}}(e^{+}) - \hat{\T{X}}(e^{-}))} 
\end{align}
where $e^{+}$ is an observed interaction $(i_1,...,i_N)$ in a set $\Omega_o$ of observed interactions, $e^{-}$ is a negative interaction with respect to $e^{+}$, and $\sigma$ is the sigmoid function.
Note that we construct a set of negative interactions randomly sampled from the set $\Omega_u$ of unobserved interactions.
Our work is complementary to existing work focusing on designing the prediction function $f$ including an effective negative sampling strategy since our proposed model is compatible with a variety of existing functions $f$.

\textbf{Training and Inference.}
We introduce how we use tensor factorization in the top-$k$ higher-order interaction prediction.
There are the training and inference phases of tensor factorization for the higher-order interaction prediction.

\begin{itemize}
	\item \textbf{Training Phase:} given positive observed interactions in training data and negative sampled interactions, a tensor factorization model learns (1) latent vectors of entities, (2) the parameters of a tensor factorization predictor that predicts high scores for the positive interactions while predicting low scores for the negative interactions.
	\item \textbf{Inference Phase:} a trained model (1) predicts scores of unobserved interactions by using the learned latent vectors and predictor, (2) sorts the scores, and (3) picks top-$k$ potential interactions based on the scores.
\end{itemize}

\begin{remark}
    There are two perspectives on tensor factorization: (1) given an interaction $(i_1,...,i_N)$, tensor factorization loads the latent vectors of the entities $i_1,...,i_N$ and predicts a value for the interaction by effectively fusing them.
    (2) Given a tensor, tensor factorization decomposes the tensor into factor matrices and parameters.
    In the first view, tensor factorization consists of encoder and predictor modules: a one-hot encoder loads the latent vectors of entities, and a predictor predicts a value using them.
    In the second view, tensor factorization serves as an encoder that projects a tensor into low-rank factor matrices and parameters.
    In this paper, we consider tensor factorization from the first perspective that it consists of encoder and predictor modules.
\end{remark}

\section{Proposed Method}
\label{sec:proposed}

In this section, we propose \textbf{\underline{T}}ensor-originated Graph \textbf{\underline{C}}onvolutional \textbf{\underline{N}}etwork (\method) which is accurate and highly compatible with tensor factorization methods for higher-order interaction prediction.
In predicting top-{\em k} higher-order interactions, traditional TF methods struggle to adequately learn latent vectors for entities since most entities have limited interactions to learn their vectors.
To overcome this limitation, our proposition is to extract a graph structure inherent in a sparse tensor and learn enriched latent vectors by neighborhood aggregation through the graph.
In Section~\ref{subsec:hypergraph}, we describe how to construct a hypergraph from a sparse tensor and its clique-expanded graph.
In Section~\ref{subsec:tcn}, we design a tensor convolutional network that outputs relation-aware latent vectors of entities of all dimensions over the clique-expanded graph.
In Section~\ref{subsec:tcn_based_tf}, we present how to integrate \method with tensor factorization methods.

\subsection{Tensor-originated Graph Construction}
\label{subsec:hypergraph}

Our goal is to extract a graph inherent in a sparse tensor for designing an expressive encoder of tensor factorization.
To obtain enriched latent vectors even for entities with limited interactions, we aggregate information from neighbor entities in the second step. 
However, we lack an effective graph structure for propagating information.
Since we utilize a model trained based on observed interactions to predict potential interactions, we need to find hidden relationships between entities from the given sparse tensor.
To find the effective graph, our proposition is to \textbf{(1) \textit{transform a sparse tensor into a hypergraph}} which can represent higher-order relationships between nodes, and then \textbf{(2) \textit{obtain a clique-expanded graph}} from the hypergraph.

\subsubsection{\textbf{Tensor-originated hypergraph.}}
We first transform a sparse tensor into a hypergraph by viewing the entities and interactions of a sparse tensor as nodes and edges of a hypergraph, respectively.
A tensor-originated hypergraph is defined as follows:
\begin{definition}[Tensor-originated Hypergraph $G_H$]
Given a set $\Omega_o$ of observed interactions in an $N$-th order sparse tensor of the size $I_1 \times \cdots \times I_N$,
a tensor-originated hypergraph $G_H = (V, E)$ consists of a set $V$ of nodes and a set $E$ of hyperedges corresponding to the entities and the observed interactions of the tensor, respectively, where the numbers $|V|$ and $|E|$ of nodes and hyperedges are equal to $I_1 + \cdots + I_N$ and $|\Omega_o|$, respectively.
The element $b_{m,n}$ of its incidence matrix $\mat{B} \in \{0,1\}^{|V|\times |E|}$ is $1$ if the $m$-th node is incident with the $n$-th hyperedge and $0$ otherwise.
In other words, the entity corresponding to the $m$-th node is included in the observed interaction corresponding to the $n$-th hyperedge.
\end{definition}
\noindent A tensor-originated hypergraph $G_H$ is an $N$-uniform hypergraph whose edges have the same size $N$.
In addition, there are no nodes corresponding to entities of the same dimension in each edge.
Figure~\ref{fig:graph_construction} illustrates an example of constructing a hypergraph.

Next, we examine the time and space complexities of constructing a hypergraph from a sparse tensor.

\begin{proposition}[Time and space complexities of hypergraph construction]
\label{PROP:time_space_complexity_hypergraph}
Assume that the number of observed interactions of a given sparse tensor is $|E|$ and the size of the tensor is $I_1  \times \cdots \times I_N$. Then, the time and space complexities for constructing a matrix of a hypergraph is $O(N|E|)$ when we use a COO (Coordinate list) format to deal with the hypergraph matrix.
\end{proposition}

\begin{proof}
    The number of edges of a hypergraph is the same number of interactions of a sparse tensor.
    For a COO format, we need to store $N|E|$ indices for the hypergraph, and hence the space and time complexities are $O(N|E|)$.
\end{proof}

The incidence matrix of a tensor-originated hypergraph is also sparse.
Note that the number $N|E|$ of nonzeros is much smaller than the size $(I_1\times\cdots\times I_N)\times |E|$ of the matrix.


\subsubsection{\textbf{Tensor-originated clique-expanded graph.}}
Next, we aim to find a clique-expanded graph from a hypergraph $G_H$.
We would utilize a tensor-originated hypergraph $G_H$ for a relation-aware encoder,
but an efficiency issue occurs since unnecessary embedding vectors of edges are generated when we use a tensor-tailored predictor after the encoder.
This is because the number of observed interactions is much larger than the sum of the number of entities for a sparse tensor:
for example, on SG data, $|\Omega_o| = 105,764$ is $12.5\times$ larger than the sum $I_1+I_2+I_3 = 8,399$ of the number of entities for a sparse tensor.
To achieve high efficiency, our proposition is to exploit clique expansion that projects a hypergraph into a pairwise graph.
Specifically, we generate a clique-expanded graph $G_C$ from $G_H$ as follows:

\begin{definition}[Tensor-originated Clique-expanded Graph $G_C$]
We define a tensor-originated clique-expanded graph as $G_C (V, E_C,$ $ w)$ where $E_C$ is a set of its edges and $w: E_C \rightarrow \mathbb{Z}$ is a weight function.
An edge weight $w(i,j)$ in $G_C$ is equal to the number of the hyperedges that include both the $m$-th node and the $n$-th node.
The adjacency matrix $\mat{A}\in \mathbb{R}^{|V| \times |V|}$ of the clique-expanded graph without self-loop where $|V|$ is equal to $I_1 + \cdots + I_N$:
\begin{align}
\label{eq:ceg}
    \mat{A} = \mat{B}\mat{B}^T - (\textrm{diag}(\mat{B}\mat{B}^T) \mathbf{1}^T) \circ \mat{I}
\end{align}
where $\textrm{diag}(\mat{B}\mat{B}^T) \in \mathbb{R}^{|V|}$ is a column vector consisting of diagonal terms of $\mat{B}\mat{B}^T$, $\mathbf{1} \in \mathbb{R}^{|V|}$ is a column vector whose all elements are $1$, $\mat{I} \in \mathbb{R}^{|V| \times |V|}$ is an identity matrix, and $\circ$ is Hadamard product.
The element $\mat{A}(i,j) (i\neq j)$ is $w(i,j)$ which is equal to the number of the interactions that include both entities corresponding to the nodes $i$ and $j$ among observed interactions in $\Omega_o$.
$\mat{A}(n,n)$ is equal to $0$.
\end{definition}

\paragraph{Example.}
Figure~\ref{fig:graph_construction} illustrates an example of constructing a hypergraph and its clique-expanded graph from a sparse tensor.
There is a $3$rd-order sparse tensor that includes three observed interactions (1,1,1), (2,2,1), and (1, 2, 2).
We construct a hypergraph that has $6$ nodes and $3$ edges and then a clique-expanded graph that has $9$ edges with $4$ $3$-cliques where $3$ $3$-cliques correspond to the observed interactions and $1$ $3$-clique (i.e., (1, 2, 1)) is an additional clique that arises from representing the tensor as the clique-expanded graph.

Next, we examine the time and space complexities of constructing a tensor-originated clique-expanded graph.

\begin{proposition}[Space complexity of clique-expanded graph construction]
\label{PROP:time_space_complexity_hypergraph}
The space complexity for constructing a tensor-originated clique-expanded graph $\mathbf{A} \in \mathbb{R}^{|V|\times |V|}$ is $O(N^2|E|)$ where $N$ is the order of a sparse tensor and $|E|$ is the number of observed interactions in a sparse tensor.
\end{proposition}

\begin{proof}
We start from computing the number of nonzeros in a $j$-th column of the matrix $\mathbf{A}$ (i.e., $\mathbf{A}[:,j]$).
We can obtain $\mathbf{A}[:,j]$ by computing $\mathbf{B}\times \mathbf{B}^T[:,j]$.
$\mathbf{A}[:,j]$ has at most $N\times n_j$ nonzeros where $n_j$ is the number of nonzeros in $\mathbf{B}^T[:,j]$.
This is because (1) we only consider the columns of $\mathbf{B}$ corresponding to the rows associated with nonzero elements in $\mathbf{B}^T[:,j]$, (2) $\mathbf{B}^T[:,j]$ has $n_j$ nonzeros, and (3) each column of $\mathbf{B}$ has $N$ nonzeros.
In other words, since each column of $\mathbf{B}$ can generate up to $N$ nonzeros in $\mathbf{A}[:,j]$, and we utilize $n_j$ columns of $\mathbf{B}$ to create nonzeros in $\mathbf{A}[:,j]$, the number of nonzeros in $\mathbf{A}[:,j]$ is at most $N\times n_j$.
Then, the space complexity of a matrix $\mathbf{A}$ is $O(\sum_{j=1}^{|V|}{N\times n_j}) = O(N\times (n_1 + ... + n_{|V|})) = O(N^2|E|)$ since $(n_1 + ... + n_{|V|})$ is equal to $N|E|$ which is the number of nonzeros in $\mathbf{B}^T$.
\end{proof}

\begin{proposition}[Time complexity of clique-expanded graph construction]
\label{PROP:time_space_complexity_hypergraph}
The time complexity for constructing a tensor-originated clique-expanded graph $\mathbf{A} \in \mathbb{R}^{|V|\times |V|}$ is $O(N^2|E|)$ where $N$ is the order of a sparse tensor and $|E|$ is the number of observed interactions in a sparse tensor.
\end{proposition}

\begin{proof}
The dominant computation in Equation~\eqref{eq:ceg} is $\mat{B}\mat{B}^T$.
As described in the space complexity, for computing $\mathbf{A}[:,j] = \mat{B}\mat{B}^T[:,j]$, we utilize (1) $\mathbf{B}^T[:,j]$ and (2) the $n_j$ columns of $\mathbf{B}$ corresponding to the rows associated with nonzero elements in $\mathbf{B}^T[:,j]$.
Note that $n_j$ is the number of nonzeros in $\mathbf{B}^T[:,j]$.
Let us denote the positions of the nonzeros in $\mathbf{B}^T[:,j]$ as $p_1,...,p_{n_j}$.
Then we construct the sub-matrix $\mathbf{B}_j \in \mathbb{R}^{|V|\times n_j}$ by selecting columns indexed by $p_1,...,p_{n_j}$ from matrix $\mat{B}$.
The resulting sub-matrix $\mathbf{B}_j$ has $N\times n_j$ non-zero entries.
Also, let a vector composed of the nonzero elements of $\mathbf{B}^T[:,j]$ be $\mathbf{b}_j \in \mathbb{R}^{n_j}$.
Then, the time complexity for $\mathbf{A}[:,j]$ is $O(N \times n_j)$ since sparse matrix-vector multiplication between the sub-matrix $\mathbf{B}_j$ and the vector $\mathbf{b}_j$ requires $O(N\times n_j)$ time.
Thus, the time complexity for $\mathbf{A}$ is $O(\sum_{j=1}^{|V|}{N\times n_j}) = O(N\times (n_1 + ... + n_{|V|})) = O(N^2|E|)$ since $(n_1 + ... + n_{|V|})$ is equal to $N|E|$ which is the number of nonzeros in $\mathbf{B}^T$.
\end{proof}

Note that $N$ is usually a very small number (e.g., $N = 3$ or $4$ in our setting).
Furthermore, $N^2$ term is the worst case in the complexities of clique-expanded graph construction and can be reduced depending on the sparsity pattern of a hypergraph.
We experimentally show the time and space costs of the clique-expanded graph construction in Section~\ref{subsec:time_and_space}.


\subsection{Tensor-originated Graph Convolutional Network (\method)}
\label{subsec:tcn}
We propose \method (Tensor-originated Graph Convolutional Network) which produces relation-aware latent vectors of entities.
To exploit a clique-expanded graph obtained from a sparse tensor, our proposition is to propagate information over the graph to enrich the latent vectors of entities.
It allows us to effectively learn the latent vectors for entities that have a limited number of interactions.
Given a tensor-originated clique-expanded graph $G_C$ and the vertical concatenation $\mat{F}^{[0]} = \vee_{n=1}^{N}(\mat{A}^{(n)}) \in \mathbb{R}^{(I_1 + \cdots + I_N) \times r}$ of factor matrices,
a relation-aware encoder outputs expressive representations of entities by effectively aggregating information of their neighbors.
Inspired by~\cite{kipf2016semi,he2020lightgcn}, we adopt a propagation rule that propagates node information layer by layer.
Given the adjacency matrix $\mat{A}$ of a tensor-originated clique-expanded graph $G_C$ and $\mat{F}^{[l]}$ of the $l$-th layer,
\method generates $\mat{F}^{[l+1]}$ of the $l+1$-th layer in Figure~\ref{fig:tcn}:
\begin{align}
    \mat{F}^{[l+1]} = \T{P}(\mat{A}, \mat{F}^{[l]})
\end{align}
where $\T{P}$ is a propagation function that propagates latent vectors of entities over a graph.
There are many GNN-based propagation functions~\cite{kipf2016semi,gilmer2017neural,hamilton2017inductive,velivckovic2017graph,wu2019simplifying,zhu2020simple,li2022g} according to~\cite{yoo2023less}, and we can choose one of them as our work is complementary to the existing works related to the GNN-based propagation functions.

\begin{figure}[t]
	\centering	
	\includegraphics[width=0.45\textwidth]{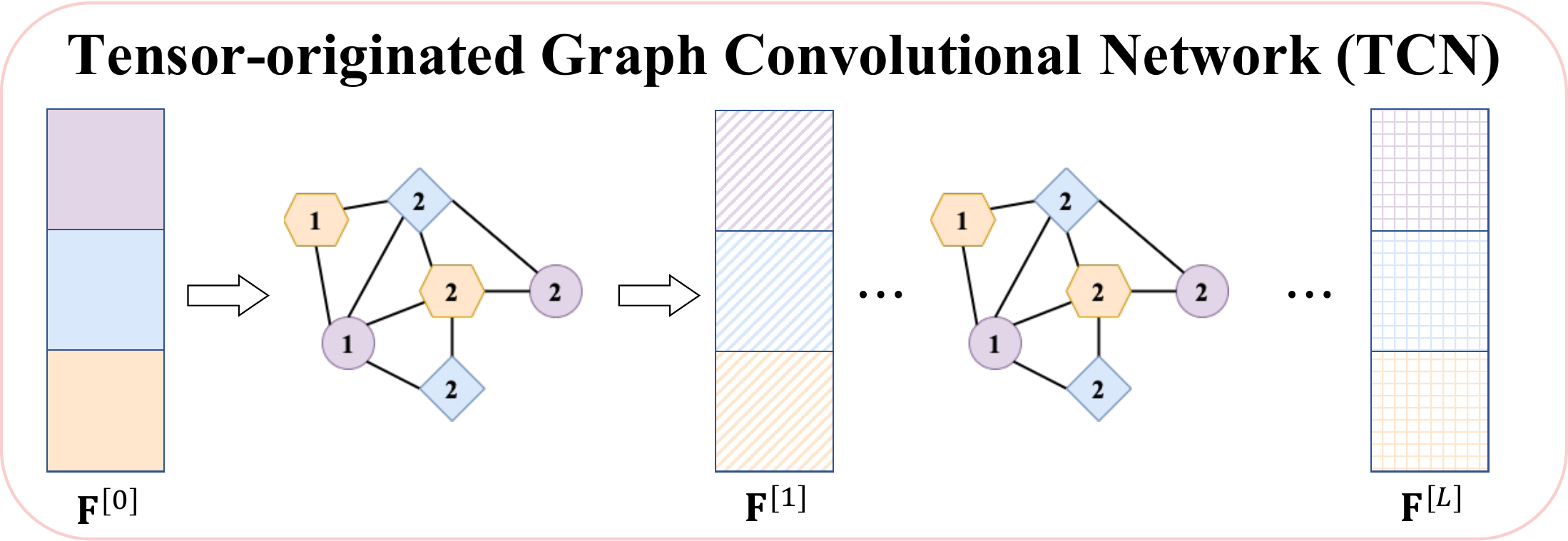}
	\caption{
	Overview of \method. 
	}
	\label{fig:tcn}
\end{figure}

When designing the propagation function $\T{P}$, it is crucial to consider that in this problem, both the parameters in the function and the input $\mat{F}^{[l]}$ are learned.
This differs from conventional graph problems that utilize non-learnable input features.
Due to the learnable input $\mat{F}^{[l]}$, a high-complexity function can cause training instability.
Thus, to mitigate it, we propose a simple but accurate linear model for propagating latent vectors over the graph as follows:
\begin{align}
\small
    \mat{F}^{[l+1]} = (\mat{D}^{-1/2}{\mat{A}}\mat{D}^{-1/2})\mat{F}^{[l]} = \tilde{\mat{A}}\mat{F}^{[l]}
\end{align}
where $\mat{D}$ is a diagonal matrix whose diagonal entry $\mat{D}(i,i)$ is equal to $\sum_{j=1}^{|V|}{\mat{A}(i,j)}$.
\begin{align}
\small
    \mat{F}^{[l+1]}(i,:) = \sum_{j\in \T{N}_i}{\frac{\mat{A}(i,j)}{\sqrt{\mat{D}(i,i)}\sqrt{\mat{D}(j,j)}}}\mat{F}^{[l]}(j,:)
\end{align}
where ${\T{N}_i}$ is a set of neighbor nodes of the $i$-th node.
In principle, we can use any existing GNN in \method, but we construct \method as a linear model $\tilde{\mat{A}}\mat{F}^{[l]}$ consisting of only feature propagation since it provides higher performance than complex models.
We will show the comparison in Section~\ref{subsec:ablation}.

\vspace{-2mm}
\subsection{Integration of \method with Tensor Factorization}
\label{subsec:tcn_based_tf}
We aim to integrate \method with tensor factorization.
Traditional TF methods focus only on effectively fusing multiple latent vectors, which may result in undertrained latent vectors in highly sparse tensors.
In contrast to them, \method prioritizes generating expressive latent vectors of entities even for a limited number of interactions, by aggregating information of neighbors of entities.
Therefore, \method effectively enhances the performance of traditional TF methods since \method and traditional TF methods serve complementary roles.
In other words, when \method and TF methods are used together, they exhibit significant synergy, yielding superior results that are up to several orders of magnitude better than those of traditional TF methods (See Table~\ref{tab:effectiveness}).
Furthermore, another benefit of \method is that the integration with \method is not limited to a specific TF method but extends to a variety of TF methods as well (See Section~\ref{subsec:compatibility}).

We consider utilizing outputs $\mat{F}^{[l]}$ for $l=0,...,L$ as an input of a tensor factorization predictor.
This contrasts with previous tensor factorization methods that use only $\mat{F}^{[0]}$ which is the vertical concatenation of initial factor matrices $\mat{A}^{(n)}$.
Since the size of each $\mat{F}^{[l]} \in \mathbb{R}^{(I_1+\cdots+I_N)\times r}$ is the same as that of $\mat{F}^{[0]}  \in \mathbb{R}^{(I_1+\cdots+I_N)\times r}$, we easily construct an input using $\mat{F}^{[l]}$ with simple operations (e.g., sum, mean, element-wise product, and concatenation).

\textbf{Sum, mean, and element-wise product.}
We first introduce how to construct an input using $\mat{F}^{[l]}$ with the simple operations, sum, mean, and element-wise product.
Let us assume that we define an input as $\mat{F}_{\textrm{final}}$.
Previous tensor factorization methods utilize $\mat{F}_{\textrm{final}}  \leftarrow \mat{F}^{[0]}$
while we obtain $\mat{F}_{\textrm{final}}$ as follows:
\begin{align}
\label{eq:general_operation}
    \mat{F}_{\textrm{final}} \leftarrow g*(\mat{F}^{[0]} \star \mat{F}^{[1]} \star \cdots \star \mat{F}^{[L]})
\end{align}
where $L$ is the number of \method layers, $g$ is a scalar value, and $\star$ represents a general operation that includes sum and element-wise product.
The sum and mean operation set $\star$ operation to the summation (i.e., $+$ operation) while those set $g$ to $1$ and $1/(L+1)$, respectively.
Or, we can choose the element-wise product for $\star$ operation.
For a given interaction $(i_1,...,i_N)$, we utilize $\mat{F}_{\textrm{final}}(i_n+S_n,:))$ instead of $\mat{A}^{(n)}(i_n,:)$ where $S_n$ is equal to $\sum_{s=0}^{n-1}{I_n}$ $(I_0 = 0)$.
The advantage of these operations is that the size of $\mat{F}_{\textrm{final}}$ is the same as that of $\mat{F}^{[0]}$, and thus we do not have to consider adjusting the hyperparameters of a tensor factorization predictor.

\textbf{Concatenation.}
Next, we introduce the concatenation operation with $\mat{F}^{[l]}$ for $l=0,...,L$.
One approach is to choose the horizontal concatenation $\|$ for $\star$ in Equation~\eqref{eq:general_operation} while setting $g$ to 1.
The size of $\mat{F}_{\textrm{final}} \leftarrow \|_{l=0}^{L}{\mat{F}^{[l]}}$ is $(I_1 + \cdots I_N) \times (L+1)r$.
Another approach is to construct $\T{F}_{\textrm{final}} \in \mathbb{R}^{(I_1+\cdots + I_N)\times r \times (L+1)}$ by stacking $\mat{F}^{[l]}$ along the third dimension:
the tensor $\T{F}_{\textrm{final}}$ is tailored for \costco and NeAT which utilize $\vee_{n=1}^{N}(\mat{A}^{(n)}(i_n,:))\in\mathbb{R}^{N\times r}$ for a given interaction $(i_1,...,i_N)$.
Then, for \costco and NeAT, we use  $\vee_{n=1}^{N}(\T{F}_{\textrm{final}}(i_n+S_n,:,:))^T\in\mathbb{R}^{(L+1)N\times r}$ as an input where $\T{F}_{\textrm{final}}(i_n+S_n,:,:))^T$ is a matrix of the size $(L+1)\times r$ and $S_n$ is equal to $\sum_{s=0}^{n-1}{I_n}$ $(I_0 = 0)$.
The concatenation operation provides a more enriched input to a tensor factorization (TF) predictor than the other operations.
However, it requires adjustment for hyperparameters of the TF predictor and increases the size of the TF predictor.
In this paper, we utilize the concatenation to construct an input from the outputs of \method and provide experimental results for comparison between the four operations in Section~\ref{subsec:ablation}.



\vspace{-2mm}
\section{Experiments}
\label{sec:experim}

\begin{table}[t!]
	\caption{Description of real-world tensor datasets.
	For each dataset, we split all observed interactions into train, valid, and test sets with the ratio 7:1:2. $(0.\#\#\#\#\%)$ indicates the density of a sparse tensor.
	}
	\centering
	\label{tab:Description}
		\resizebox{0.49\textwidth}{!}{
	\begin{tabular}{lrrrrr}
		\toprule
		\textbf{Dataset} & \textbf{Tensor Size} & \textbf{$\#$ observed interactions} & \textbf{$\#$ unobserved interactions}\\
		\midrule
		 ML-100k\footnoteref{foot:ml} & $943\times 1,682 \times 31 \times 8$ & $100,000$ $(0.0254\%)$ & $393,259,248$ \\ 	
		 SG\footnoteref{foot:sg}~\cite{li2015rank} & $2321\times 5596 \times 482$ & $105,764$ $(0.0016\%)$ & $6,260,326,006$\\
		 Gowalla\footnoteref{foot:gowalla}~\cite{cho2011friendship} & $9,652 \times 12,705 \times 296$ & $842,173$ $(0.0023\%)$ & $36,297,241,187$\\   
		 UCI\footnoteref{foot:uci} & $1,350\times 1,862 \times 7$ & $22,640$ $(0.1286\%)$ & $17,573,260$\\		 
		 UMLS\footnoteref{foot:umls} & $135\times 46 \times 132$ & $6,529$ $(0.7964\%)$ & $817,109$\\	
		 NYC\footnoteref{foot:nyc} & $1,084\times 38,334 \times 7,641$ & $225,701$ $(0.00007\%)$ & $317,514,316,195$ \\			
		 Yelp\footnoteref{foot:yelp} & $70,818\times 15,580 \times 109$ & $333,481$ $(0.00027\%)$ & $120,264,210,479$\\		
		 Yahoo\footnoteref{foot:yahoo} & $82,309\times 82,308 \times 168$ & $785,749$ $(0.00007\%)$ & $1,138,146,995,147$ \\
		\bottomrule
	\end{tabular}}
\end{table}

In this section, we aim to answer the following questions:
\begin{itemize}
	\item[Q1] \textbf{Compatibility of \method (Section~\ref{subsec:compatibility}).} Is \method applicable to various tensor factorization methods effectively?
	\item[Q2] \textbf{Performance (Section~\ref{subsec:performance}).} To what extent does \method achieve higher performance than competitors? 
	\item[Q3] \textbf{Ablation study (Section~\ref{subsec:ablation}).} How much do the operations used (or not used) in \method affect its performance?
	\item[Q4] \textbf{Impact of \method for orders (Section~\ref{subsec:impact_order}).} How does the impact of \method vary depending on different tensor orders?
        \item[Q5] \textbf{Computational time of \method (Section~\ref{subsec:time}).} How much does \method provoke overhead for computational time? 
        \item[Q6] \textbf{Time and Space Requirements for Graphs (Section~\ref{subsec:time_and_space}).} How much time and space do we require for constructing hypergraphs and clique-expanded graphs?
\end{itemize}

\begin{table*}[!h]
\caption{Compatibility of \method for AP@$k$.
The triangles \red{$\blacktriangle$} and \blue{$\blacktriangledown$} indicate performance improvement and degradation, respectively.
A bold text indicates the best performance.
With a few exceptions, \method further near-universally improves the performance of existing tensor factorization methods.
$-$ indicates a value less than $0.0001$.
}
\label{tab:effectiveness}
\resizebox{0.99\textwidth}{!}{
\begin{tabular}{ll|rr|rr|rr|rr|rr|rr|r}
\toprule
\multicolumn{2}{c|}{Setting}                              & \multicolumn{2}{c|}{CP}                               & \multicolumn{2}{c|}{Tucker}                                   & \multicolumn{2}{c|}{NCFT} & \multicolumn{2}{c|}{M$^2$DMTF}   & \multicolumn{2}{c|}{NeAT}            & \multicolumn{2}{c|}{CostCo}                                   & \multicolumn{1}{c}{\multirow{2}{*}{\begin{tabular}[c]{@{}c@{}}Average\\ Improvement\end{tabular}}} \\
\multicolumn{1}{c}{Dataset} & \multicolumn{1}{c|}{Metric} & 
\multicolumn{1}{c}{w/o \method} & \multicolumn{1}{c|}{w/ \method} & 
\multicolumn{1}{c}{w/o \method} & \multicolumn{1}{c|}{w/ \method} & 
\multicolumn{1}{c}{w/o \method} & \multicolumn{1}{c|}{w/ \method} & 
\multicolumn{1}{c}{w/o \method} & \multicolumn{1}{c|}{w/ \method} & 
\multicolumn{1}{c}{w/o \method} & \multicolumn{1}{c|}{w/ \method} & 
\multicolumn{1}{c}{w/o \method} & \multicolumn{1}{c|}{w/ \method} & \multicolumn{1}{c}{}                                     \\
\midrule
\multirow{3}{*}{ML-100k} 
& AP@100                      & 0.0271                 &       0.0275\red{$\blacktriangle$}                   & 0.0566                     & 0.1309\red{$\blacktriangle$}                           & 0.0553                   &         0.1586\red{$\blacktriangle$}              & 0.0459 & 0.0943\red{$\blacktriangle$} & 
	0.0806 &  \textbf{0.2532} \red{$\blacktriangle$}    &
0.0845                     & 0.1309\red{$\blacktriangle$}                          & 115.6\%                                                   \\
& AP@1000                     & 0.0206                 &       0.0150\blue{$\blacktriangledown$}                   & 0.0355                     & 0.0705\red{$\blacktriangle$}                           & 0.0352                   &        0.1020\red{$\blacktriangle$}                & 0.0210 & 0.0568\red{$\blacktriangle$} &
	0.0401 & 0.0879 \red{$\blacktriangle$} &
 0.0509                     & \textbf{0.1048}\red{$\blacktriangle$}                          & 109.3\%                                                  \\
& AP@10000                    & 0.0101                 &      0.0071\blue{$\blacktriangledown$}                 & 0.0165                     & 0.0301\red{$\blacktriangle$}                           & 0.0162                   &       0.0392\red{$\blacktriangle$}                  & 0.0064 & 0.0218\red{$\blacktriangle$} &
	0.0176 & 0.0313 \red{$\blacktriangle$} &
 0.0255                     & \textbf{0.0521}\red{$\blacktriangle$}                          & 123.4\%     \\
\midrule
\midrule
\multirow{3}{*}{SG}
& AP@100                      & 0.0641                 &     0.0457\blue{$\blacktriangledown$}                   & 0.0252                     & 0.0424\red{$\blacktriangle$}                           & 0.0087                   &        0.0455\red{$\blacktriangle$}         & 0.0074 &   0.0206\red{$\blacktriangle$}      &
	0.0285 & \textbf{0.0832}\red{$\blacktriangle$}  &
 0.0340                    & 0.0517\red{$\blacktriangle$}                          & 147.4\%                                                   \\
& AP@1000                     & 0.0199                 &     0.0148\blue{$\blacktriangledown$}                   & 0.0102                     & 0.0200\red{$\blacktriangle$}                           & 0.0042                   &        0.0191\red{$\blacktriangle$}                 & 0.0028 & 0.0088\red{$\blacktriangle$} & 
	0.0100 &  \textbf{0.0302}\red{$\blacktriangle$} &
0.0130                     & 0.0266\red{$\blacktriangle$}                          & 157.7\%                                                  \\
& AP@10000                    & 0.0050                 &     0.0042\blue{$\blacktriangledown$}                   & 0.0028                     & 0.0069\red{$\blacktriangle$}                           & 0.0017                   &        0.0068\red{$\blacktriangle$}                  & 0.0007 & 0.0024\red{$\blacktriangle$} & 
	0.0031 & 0.0077\red{$\blacktriangle$}  &
0.0053                     & \textbf{0.0084}\red{$\blacktriangle$}                          & 76.80\%                  \\     
\midrule
\midrule       
\multirow{3}{*}{Gowalla}
& AP@100                      & 0.0003                 &     0.0018\red{$\blacktriangle$}                   & 0.0048                     & 0.0131\red{$\blacktriangle$}                           & -                   &        0.0254\red{$\blacktriangle$}         & - &   0.0126\red{$\blacktriangle$}      & 
	0.0054 &  0.0173\red{$\blacktriangle$} &
0.0307                     & \textbf{0.0662}\red{$\blacktriangle$}                          & 252.2\%                                                   \\
& AP@1000                     & 0.0001                 &     0.0010\red{$\blacktriangle$}                   & 0.0007                     & 0.0052\red{$\blacktriangle$}                           & 0.0001                   &        0.0103\red{$\blacktriangle$}                 & - & 0.0040\red{$\blacktriangle$} & 
	0.0018 &  0.0083\red{$\blacktriangle$} &
0.0119                     & \textbf{0.0260}\red{$\blacktriangle$}                          & 592.8\%                                                  \\
& AP@10000                    & -                 &     0.0003\red{$\blacktriangle$}                   & 0.0001                     & 0.0012\red{$\blacktriangle$}                           & 0.0002                   &        0.0041\red{$\blacktriangle$}                  & - & 0.0007\red{$\blacktriangle$} & 
	0.0013 & 0.0056\red{$\blacktriangle$}  &
0.0054                     & \textbf{0.0156}\red{$\blacktriangle$}                          & 1404.9\%                  \\     
\midrule
\midrule                           
\multirow{3}{*}{UCI}
                            & 
AP@40                      & 0.0276                 &        0.0385\red{$\blacktriangle$}               & 0.0399                     & 0.0300\blue{$\blacktriangledown$}                           & 0.0149                   &                 0.0266\red{$\blacktriangle$}        & 0.0042 & 0.0193\red{$\blacktriangle$}  &
	0.0250 & \textbf{0.0733} \red{$\blacktriangle$}   &
 0.0085                     & 0.0436\red{$\blacktriangle$}                          & 176.4\%                                                   \\
                            &
AP@400                     & 0.0115                 &        0.0140\red{$\blacktriangle$}                & 0.0136                     & 0.0146\red{$\blacktriangle$}                           & 0.0059                   &               0.0108\red{$\blacktriangle$}          
& 0.0025 & 0.0111\red{$\blacktriangle$} & 
	0.0052 & \textbf{0.0201} \red{$\blacktriangle$}  &
0.0036                     & 0.0168\red{$\blacktriangle$}                          & 184.8\%                                                  \\
                            &
AP@4000                    & 0.0045                 &        \textbf{0.0060}\red{$\blacktriangle$}                & 0.0049                     & 0.0054\red{$\blacktriangle$}                           & 0.0022                   &               0.0043\red{$\blacktriangle$}        & 0.0015 & 0.0046\red{$\blacktriangle$}   &
	0.0020 & \textbf{0.0060} \red{$\blacktriangle$} &
 0.0031                     & 0.0059\red{$\blacktriangle$}                          & 108.7\%                  \\     
\midrule
\midrule    
\multirow{3}{*}{UMLS}            
                            & 
AP@200                      & 0.2848                 &       0.2966\red{$\blacktriangle$}                & 0.3577                     & 0.3772\red{$\blacktriangle$}                           & 0.3111                   &          0.3656\red{$\blacktriangle$}             
 & 0.3111 & 0.3652\red{$\blacktriangle$} & 
 	0.2968 & 0.3382\red{$\blacktriangle$}   &
 0.3167                     & \textbf{0.3877}\red{$\blacktriangle$}                          & 13.47\%                                                   \\
                            &
AP@600                     & 0.2632                 &        0.2879\red{$\blacktriangle$}                & 0.3391                     & 0.3614\red{$\blacktriangle$}                           & 0.2759                   &          0.3437\red{$\blacktriangle$}                
& 0.2759 & 0.3305\red{$\blacktriangle$} & 
	0.2741 &  0.3287\red{$\blacktriangle$} &
0.2967                     & \textbf{0.3658}\red{$\blacktriangle$}                          & 17.25\%                                                  \\
                            &
AP@1000                    & 0.2468                 &        0.2680\red{$\blacktriangle$}                & 0.3069                     & 0.3403\red{$\blacktriangle$}                           & 0.2523                   &          0.3120\red{$\blacktriangle$}               
& 0.2523 & 0.3098\red{$\blacktriangle$} & 
	0.2489 & 0.3005\red{$\blacktriangle$}  &
0.2734                     & \textbf{0.3439}\red{$\blacktriangle$}                          & 18.74\%                  \\     
\midrule
\midrule  
\multirow{3}{*}{NYC}            
                            & 
AP@100                      & 0.9562                 &       0.3584\blue{$\blacktriangledown$}                & 0.7343                     & 0.7137\blue{$\blacktriangledown$}                           & 0.5015                   &          0.6977\red{$\blacktriangle$}             
 & 0.3240 & 0.7734\red{$\blacktriangle$} & 
 	0.3892 & \textbf{0.9642}\red{$\blacktriangle$}   &
 0.9080                     & 0.9602\red{$\blacktriangle$}                          & 44.33\%                                                   \\
                            &
AP@1000                     & 0.8536                 &        0.2892\blue{$\blacktriangledown$}                & 0.6038                     & 0.6514\red{$\blacktriangle$}                           & 0.4123                   &          0.7144\red{$\blacktriangle$}                
& 0.2396 & 0.7037\red{$\blacktriangle$} & 
	0.2827 &  \textbf{0.9458}\red{$\blacktriangle$} &
0.8868                     & 0.9362\red{$\blacktriangle$}                          & 73.42\%                                                  \\
                            &
AP@10000                    & 0.3233                 &        0.2146\blue{$\blacktriangledown$}                & 0.2673                     & 0.4847\red{$\blacktriangle$}                           & 0.2112                   &          0.5327\red{$\blacktriangle$}               
& 0.1074 & 0.3965\red{$\blacktriangle$} & 
	0.1298 & 0.7466\red{$\blacktriangle$}  &
0.7023                     & \textbf{0.8294}\red{$\blacktriangle$}                          & 157.2\%                  \\    
\midrule
\midrule
\multirow{3}{*}{Yelp}            
                            & 
AP@100                      & 0.2945                 &       0.3456\red{$\blacktriangle$}                & 0.4099                     & 0.6020\red{$\blacktriangle$}                           & 0.2834                   &          \textbf{0.6750}\red{$\blacktriangle$}             
 & 0.4059 & 0.6260\red{$\blacktriangle$} & 
 	0.3976 & 0.5392\red{$\blacktriangle$}   &
 0.4902                     & 0.6084\red{$\blacktriangle$}                          & 52.72\%                                                   \\
                            &
AP@1000                     & 0.1741                 &        0.2655\red{$\blacktriangle$}                & 0.2894                     & 0.5002\red{$\blacktriangle$}                           & 0.2045                   &          \textbf{0.5591}\red{$\blacktriangle$}                
& 0.3078 & 0.5015\red{$\blacktriangle$} & 
	0.2748 &  0.3989\red{$\blacktriangle$} &
0.3870                     & 0.4949\red{$\blacktriangle$}                          & 72.45\%                                                  \\
                            &
AP@10000                    & 0.0763                 &        0.1436\red{$\blacktriangle$}                & 0.1430                     & 0.2706\red{$\blacktriangle$}                           & 0.1138                   &          \textbf{0.3238}\red{$\blacktriangle$}               
& 0.1608 & 0.2640\red{$\blacktriangle$} & 
	0.1388 & 0.2215\red{$\blacktriangle$}  &
0.2022                     & 0.2696\red{$\blacktriangle$}                          & 188.4\%                  \\    
\midrule
\midrule
\multirow{3}{*}{Yahoo}            
                            & 
AP@10000                      & 0.1912                 &       0.1335\blue{$\blacktriangledown$}                & 0.2630                     & 0.6364\red{$\blacktriangle$}                           & 0.3664                   &          0.8746\red{$\blacktriangle$}             
 & 0.0932 & 0.5916\red{$\blacktriangle$} & 
 	0.0239 & 0.9291\red{$\blacktriangle$}   &
 0.9307                     & \textbf{0.9538}\red{$\blacktriangle$}                          & 762.7\%                                                   \\
                            &
AP@50000                     & 0.1491                 &        0.1164\blue{$\blacktriangledown$}                & 0.1946                     & 0.4289\red{$\blacktriangle$}                           & 0.2216                   &          0.7766\red{$\blacktriangle$}                
& 0.0470 & 0.4173\red{$\blacktriangle$} & 
	0.0141 &  0.7892\red{$\blacktriangle$} &
0.8698                     & \textbf{0.9216}\red{$\blacktriangle$}                          & 1106\%                                                  \\
                            &
AP@100000                    & 0.1190                 &        0.0985\blue{$\blacktriangledown$}                & 0.1517                     & 0.3091\red{$\blacktriangle$}                           & 0.1555                   &          0.6357\red{$\blacktriangle$}               
& 0.0329 & 0.3050\red{$\blacktriangle$} & 
	0.0105 & 0.5729\red{$\blacktriangle$}  &
0.7185                     & \textbf{0.8362}\red{$\blacktriangle$}                          & 1099\%                  \\    
\bottomrule
\end{tabular}}
\end{table*}

\textbf{Datasets.}
We use eight real-world datasets summarized in Table~\ref{tab:Description}:
ML-100k{\footnote{Available at \url{https://grouplens.org/datasets/movielens/}\label{foot:ml}}}, 
SG{\footnote{Available at \url{https://github.com/USC-Melady/KDD19-CoSTCo}\label{foot:sg}}}~\cite{li2015rank},
Gowalla{\footnote{Available at \url{https://snap.stanford.edu/data/loc-gowalla.html}\label{foot:gowalla}}}~\cite{cho2011friendship},
UCI{\footnote{Available at \url{https://github.com/chengw07/NetWalk/tree/master}\label{foot:uci}}}~\cite{opsahl2009clustering},
UMLS{\footnote{Available at \url{https://github.com/TimDettmers/ConvE/tree/master?tab=readme-ov-file}\label{foot:umls}}},
NYC{\footnote{Available at \url{https://sites.google.com/site/yangdingqi/home/foursquare-dataset}\label{foot:nyc}}},
Yelp{\footnote{Available at \url{https://www.yelp.com/dataset}\label{foot:yelp}}}, and
Yahoo{\footnote{Available at \url{https://webscope.sandbox.yahoo.com/catalog.php?datatype=r}\label{foot:yahoo}}} datasets.
ML-100k data is movie recommendation data of the form (user, movie, day, month).
SG data and Gowalla data are Point-of-interest (POI) recommendation data whose form is (user, location, time).
For Gowalla data, we filtered out entities (i.e., user, location, time) that have fewer than 30 interactions.
UCI data contains message-sending relationships between students at the University of California.
It has the form (sender, receiver, time).
UMLS is a medical knowledge graph of the form (subject, relation, object).
NYC~\cite{yang2014modeling} is a check-in dataset of the form (user, venue, time).
Yelp and Yahoo are recommendation data of the form (user, business, time) and (user, music, time), respectively.
We randomly split observed interactions into train, validation, and test data with the ratio 7:1:2.

\textbf{Hyperparameter setting.}
We set the batch size and the embedding size of entities to $256$ and $10$, respectively.
We use the number of epochs as $20$ on SG, Gowalla, NYC, Yelp, and Yahoo datasets while using the number of epochs as $100$ on the other datasets. 
We use the following learning rates and weight decays, respectively: $[10^{-1}, 10^{-2}, 10^{-3}]$ and $[0, 10^{-3}, 10^{-5}]$.
We learn all models using the Adam optimizer~\cite{kingma2014adam}.
We perform $L=2$ layer propagation in \method, and use a ReLU function~\cite{fukushima1975cognitron} as an activation function.
We test the performance of each model by running it five times and report the average.

\textbf{Competitors.}
We compare \method with existing tensor factorization methods, MLP (Multi-Layer Perceptron), and hypergraph-based methods: \textbf{CP decomposition~\cite{kiers2000towards}}, \textbf{Tucker decomposition~\cite{tucker1966some}}, \textbf{NCF~\cite{he2017neural} extension for tensor (NCFT)} which extends a neural network-based collaborative filtering (NCF) method to a tensor, \textbf{\costco~\cite{liu2019costco}}, \textbf{M$^2$DMTF~\cite{fan2021multi}}, and \textbf{NeAT~\cite{ahn2024neat}}.
We also compare \method with \textbf{MLP} and \textbf{HyperConv}{\footnote{\url{https://pytorch-geometric.readthedocs.io/en/latest/generated/torch_geometric.nn.conv.HypergraphConv.html}}~\cite{bai2021hypergraph} which is a hypergraph convolution with a 2-layer MLP predictor for hyperedge prediction.
We implement all the competitors using PyTorch.

\textbf{Integration of \method with a TF method.}
In Sections~\ref{subsec:performance},~\ref{subsec:ablation} (except for varied GNN types), and~\ref{subsec:time}, we use \method + CostCo since \method + CostCo achieves better performance than \method + other TF methods as shown in Section~\ref{subsec:compatibility}.

\begin{table*}[!t]
\caption{Performance. \method + CostCo outperforms competitors except for the case, AP@100 on SG data. We report the average and the standard deviation for \method. Bold and underlined texts indicate the best and the second-best performance, respectively.
$-$ indicates a value less than $0.0001$.}
\label{tab:performance}
\resizebox{0.999\textwidth}{!}{
\begin{tabular}{c|rrrrrrrrrrrrrrrr}
\toprule
\multirow{2}{*}{Method}                                         & \multicolumn{2}{c}{ML-100k}                               & \multicolumn{2}{c}{SG}                                    & \multicolumn{2}{c}{Gowalla}                               & \multicolumn{2}{c}{UCI}                                   & \multicolumn{2}{c}{UMLS}   
	& \multicolumn{2}{c}{NYC}                                   & \multicolumn{2}{c}{Yelp}                                 
	& \multicolumn{2}{c}{Yahoo}                                   \\  
\cmidrule(lr){2-3} \cmidrule(lr){4-5} \cmidrule(lr){6-7} \cmidrule(lr){8-9} \cmidrule(lr){10-11} \cmidrule(lr){12-13} \cmidrule(lr){14-15} \cmidrule(lr){16-17}
                                                                & \multicolumn{1}{l}{AP@100} & \multicolumn{1}{l}{AP@10000} & \multicolumn{1}{l}{AP@100} & \multicolumn{1}{l}{AP@10000} & \multicolumn{1}{l}{AP@100} & \multicolumn{1}{l}{AP@10000} & \multicolumn{1}{l}{AP@40} & \multicolumn{1}{l}{AP@4000} & \multicolumn{1}{l}{AP@200} & \multicolumn{1}{l}{AP@1000} &
	                                                                \multicolumn{1}{l}{AP@1000} & \multicolumn{1}{l}{AP@10000} &
	                                                                \multicolumn{1}{l}{AP@1000} & \multicolumn{1}{l}{AP@10000} &
	                                                                 \multicolumn{1}{l}{AP@50000} & \multicolumn{1}{l}{AP@100000} \\
                                                                \midrule
\multicolumn{1}{c|}{CP}                                                              & 0.0271                     & 0.0101                       & \textbf{0.0641}                     & 0.0050                       & 0.0003                     & -                       & 0.0276                     & 0.0045                       & 0.2848                     & 0.2468     
& 0.8536 & 0.3233 & 0.0747 & 0.0307 & 0.1491 & 0.1190 \\
\multicolumn{1}{c|}{Tucker}                                                          & 0.0566                     & 0.0165                       & 0.0252                     & 0.0028                       & 0.0048                     & 0.0001                       & \underline{0.0399}                     & 0.0049                       & 0.3577                     & \underline{0.3069}                       
& 0.6038 & 0.2673 & 0.1741 & 0.0763 & 0.1946 & 0.1517 \\
\multicolumn{1}{c|}{NCFT}                                                            & 0.0553                     & 0.0162                       & 0.0087                     & 0.0017                       & -                          & 0.0002                            & 0.0149                     & 0.0022                       & 0.3111                     & 0.2523                       
& 0.4123 & 0.2112 & 0.2045 & 0.1138 & 0.2216 & 0.1555 \\
\multicolumn{1}{c|}{M$^2$DMTF}                                                          & 0.0459                     & 0.0064                       & 0.0074                     & 0.0007                       & -                          & -                            & 0.0042                     & 0.0015                       & 0.2774                     & 0.1698                       
& 0.2396 & 0.1074 & 0.3078 & 0.1608 & 0.0470 & 0.329 \\
\multicolumn{1}{c|}{MLP}                                                             & 0.0789                     & 0.0173                       & 0.0246                     & 0.0046                       & 0.0240                     & 0.0020                       & 0.0165                     & 0.0028                       & \underline{0.3598}                     & 0.2807                       
& 0.7156 & 0.4261 & \underline{0.4378} & \underline{0.2342} & 0.5130 & 0.3649 \\
\multicolumn{1}{c|}{HyperGCN}                                                        & 0.0623                     & 0.0121                       & 0.0111                     & 0.0019                       & 0.0042                     & 0.0008                       & 0.0341                     & \underline{0.0054}                       & 0.3290                     & 0.2790                      
& 0.7646 & 0.4509 & 0.3961 & 0.2106 & 0.5180 & 0.3816\\
\multicolumn{1}{c|}{CostCo}                                                          & \underline{0.0845}                     & \underline{0.0255}                       & 0.0340                     & \underline{0.0053}                       & \underline{0.0307}                     & \underline{0.0054}                       & 0.0085                     & 0.0031                       & 0.3167                     & 0.2734                       
& \underline{0.8868} & \underline{0.7023} & 0.3870 & 0.2022 & \underline{0.8698} & \underline{0.7185} \\
\multicolumn{1}{c|}{NeAT}                                                          & 0.0806                     & 0.0176                       &    0.0285                  & 0.0031                       &  0.0054                    & 0.0013                       & 0.0250                     & 0.0020                       & 0.2968                     & 0.2489                       
& 0.2827 & 0.1298 & 0.2748 & 0.1388 & 0.0140 & 0.0105 \\
\midrule
\rowcolor{beige}
 \begin{tabular}[c]{@{}c@{}}\textbf{TCN+CostCo}\\ \textbf{(Proposed)} \end{tabular} &     \begin{tabular}[r]{@{}r@{}}\textbf{0.1309}\\ $\pm$ \textbf{0.037}\end{tabular}    &     \begin{tabular}[r]{@{}r@{}}\textbf{0.0521}\\ $\pm$ \textbf{0.005}\end{tabular}    & \begin{tabular}[r]{@{}r@{}}\underline{0.0517}\\ $\pm$ {0.018}\end{tabular}    &     \begin{tabular}[r]{@{}r@{}}\textbf{0.0084}\\ $\pm$ \textbf{0.002}\end{tabular}    &       \begin{tabular}[r]{@{}r@{}}\textbf{0.0662}\\ $\pm$ \textbf{0.047}\end{tabular}    &          \begin{tabular}[r]{@{}r@{}}\textbf{0.0156}\\ $\pm$ \textbf{0.002}\end{tabular}    &          \begin{tabular}[r]{@{}r@{}}\textbf{0.0436}\\ $\pm$ \textbf{0.008}\end{tabular}    &       \begin{tabular}[r]{@{}r@{}}\textbf{0.0059}\\ $\pm$ \textbf{0.001}\end{tabular}    &          \begin{tabular}[r]{@{}r@{}}\textbf{0.3877}\\ $\pm$ \textbf{0.039}\end{tabular}    &   \begin{tabular}[r]{@{}r@{}}\textbf{0.3439}\\ $\pm$ \textbf{0.008}\end{tabular}        
 &   \begin{tabular}[r]{@{}r@{}}\textbf{0.9362}\\ $\pm$ \textbf{0.009}\end{tabular}
 &   \begin{tabular}[r]{@{}r@{}}\textbf{0.8294}\\ $\pm$ \textbf{0.018}\end{tabular}
 &   \begin{tabular}[r]{@{}r@{}}\textbf{0.4949}\\ $\pm$ \textbf{0.074}\end{tabular}
 &   \begin{tabular}[r]{@{}r@{}}\textbf{0.2696}\\ $\pm$ \textbf{0.030}\end{tabular}
  &   \begin{tabular}[r]{@{}r@{}}\textbf{0.9216}\\ $\pm$ \textbf{0.006}\end{tabular} 
    &   \begin{tabular}[r]{@{}r@{}}\textbf{0.8362}\\ $\pm$ \textbf{0.008}\end{tabular} \\
 \midrule
 \rowcolor{beige}
 \textbf{Improvement} & \textbf{54.9}\% & \textbf{104.3\%} & -19.3\% & \textbf{58.4\%} & \textbf{215.6\%} & \textbf{288.9\%} & \textbf{9.27\%} & \textbf{9.25\%} & \textbf{7.75\%} & \textbf{12.0\%} 
 & \textbf{5.57\%} & \textbf{18.1\%} & \textbf{13.0\%} & \textbf{15.1\%}  & \textbf{5.98\%} & \textbf{16.3\%}  \\
\bottomrule
\end{tabular}}
\end{table*}

\textbf{Metrics.}
We use the following metric to evaluate the performance of \method and its competitors: AP@$k$ (Average Precision@$k$).
\begin{align}
        \small
		AP@k = \frac{1}{min(k, |\Omega_{test}|)} \sum_{i=1}^{k}{\text{Precision@i} * rel(i)}
	\end{align}
where $\text{Precision}@i = \frac{|\Omega_{prediction,i} \cap \Omega_{test}|}{i}$, $\Omega_{test}$ is a testset, $\Omega_{prediction,i}$ is a set of top-$i$ predicted interactions, and 
\begin{align}
    rel(i) = \begin{cases}
  1 & \text{if $i$-th ranked interaction is in $\Omega_{test}$} \\
  0 & \text{otherwise}.
\end{cases}
\end{align}
AP@$k$ not only considers the number of test interactions but also ranks of them in top-$k$ interactions.
%

\textbf{Evaluation Procedure.}
As described in Section~\ref{subsec:hoir}, we first train a model using the observed interactions in the training set.
For datasets such as ML-100k, SG, Gowalla, UCI, and UMLS, we predict the scores for all unobserved interactions using the trained model, select the top-$k$ interactions with the highest scores, and evaluate AP@$k$ using a test set and the top-$k$ interactions.
For NYC, Yelp, and Yahoo datasets, we predict the scores for the unobserved interactions in a candidate set that consists of test interactions and randomly sampled interactions due to a large number of unobserved interactions (i.e., 317 billion, 120 billion, and 1138 billion).
Note that the size of the candidate set is $1000\times$ larger than the number of test interactions.
We evaluate AP@$k$ based on these scores.

\vspace{-2mm}
\subsection{Compatibility of \method (Q1)}
\label{subsec:compatibility}
We evaluate the compatibility of \method by adapting it to various tensor factorization methods.
Table~\ref{tab:effectiveness} shows that \method is highly applicable to existing tensor factorization methods.
\method improves the performance of traditional TF methods for all $k$ on all datasets except for a few cases: (1) CP decomposition on ML-100k, SG,  NYC, and Yahoo data and (2) Tucker decomposition for AP@40 on UCI data and AP@100 on NYC data.
\method allows traditional TF methods to utilize enriched latent vectors that are generated by the aggregation of information from neighbor entities.
For Gowalla and Yahoo datasets, which are large tensors, the performance gain from \method is substantial, with improvements of up to 1404.9\%.
This is because NCFT, M$^2$DMTF, and NeAT fail to predict top-$k$ potential interactions without \method.
\method helps these methods prevent prediction failures by generating expressive latent vectors.
The average improvements are lower on UMLS data than the other datasets since the densities of UMLS data are relatively high (i.e., $0.7964\%$).
\method degrades the performance of CP decomposition on ML-100k, SG,  NYC, and Yahoo data since its predictor is too simple to capture complex patterns from the output of \method.
Note that the predictor of CP decomposition consists only of element-wise product and summation without any learnable parameters.
Nevertheless, \method remains effective for CP decomposition on the other datasets, demonstrating its compatibility and impact.

\begin{table}[!t]
\caption{Ablation study for feature transformation and nonlinear activation function.
See detail in Section~\ref{subsec:ablation}.}
\label{tab:ablation_fn}
\resizebox{0.49\textwidth}{!}{
\begin{tabular}{c|rrrrrr}
\toprule
\multirow{2}{*}{Method}                                         & \multicolumn{2}{c}{ML-100k}                               & \multicolumn{2}{c}{SG}                                    & \multicolumn{2}{c}{Gowalla}                                                              \\  
\cmidrule(lr){2-3} \cmidrule(lr){4-5} \cmidrule(lr){6-7}
                                                                & \multicolumn{1}{l}{AP@100} & \multicolumn{1}{l}{AP@10000} & \multicolumn{1}{l}{AP@100} & \multicolumn{1}{l}{AP@10000} & \multicolumn{1}{l}{AP@100} & \multicolumn{1}{l}{AP@10000}  \\
                                                                \midrule
\multicolumn{1}{c|}{TCN w/ F}                                                              & \textbf{0.1585}                     & 0.0406                       & \underline{0.0537}                     & 0.0071                       & \underline{0.0267}                     & \underline{0.0089}                                          \\
\multicolumn{1}{c|}{TCN w/ N}                                                          & 0.1257                      & \underline{0.0407}                       & 0.0455                     & \underline{0.0087}                       & 0.0228                     & 0.0087                                             \\
\multicolumn{1}{c|}{TCN w/ F+N}                                                            & 0.0880                     & 0.0309                       & \textbf{0.0601}                     & \textbf{0.0092}                       & 0.0182                          & 0.0055                                              \\
\midrule
\rowcolor{beige}
 \begin{tabular}[c]{@{}c@{}}\textbf{TCN} \end{tabular} &     \begin{tabular}[r]{@{}r@{}}\underline{0.1309}\\ $\pm$ {0.037}\end{tabular}    &     \begin{tabular}[r]{@{}r@{}}\textbf{0.0521}\\ $\pm$ \textbf{0.005}\end{tabular}    & \begin{tabular}[r]{@{}r@{}}{0.0517}\\ $\pm$ {0.018}\end{tabular}    &     \begin{tabular}[r]{@{}r@{}}{0.0084}\\ $\pm$ {0.002}\end{tabular}    &       \begin{tabular}[r]{@{}r@{}}\textbf{0.0662}\\ $\pm$ \textbf{0.047}\end{tabular}    &          \begin{tabular}[r]{@{}r@{}}\textbf{0.0156}\\ $\pm$ \textbf{0.002}\end{tabular}    \\
\bottomrule
\end{tabular}}
\end{table}

\begin{table}[!t]
\caption{Ablation study for operations used in integrating \method with a traditional TF method.
See detail in Section~\ref{subsec:ablation}.}
\label{tab:ablation_hmp}
\resizebox{0.49\textwidth}{!}{
\begin{tabular}{c|rrrrrr}
\toprule
\multirow{2}{*}{Method}                                         & \multicolumn{2}{c}{ML-100k}                               & \multicolumn{2}{c}{SG}                                    & \multicolumn{2}{c}{Gowalla} \\  
\cmidrule(lr){2-3} \cmidrule(lr){4-5} \cmidrule(lr){6-7}
                                                                & \multicolumn{1}{l}{AP@100} & \multicolumn{1}{l}{AP@10000} & \multicolumn{1}{l}{AP@100} & \multicolumn{1}{l}{AP@10000} & \multicolumn{1}{l}{AP@100} & \multicolumn{1}{l}{AP@10000}  \\
                                                                \midrule
\multicolumn{1}{c|}{TCN w/ Sum}                                                              & \underline{0.0976}                     & 0.0200                       & 0.0392                     & 0.0062                       & \underline{0.0369}                     & \underline{0.0095}                         \\
\multicolumn{1}{c|}{TCN w/ Mean}                                                          & 0.0724                     & 0.0201                       & \underline{0.0499}                     & \underline{0.0067}                       & 0.0285                     & 0.0090                                  \\
\multicolumn{1}{c|}{TCN w/ Product}                                                            & \underline{0.0976}                     & \underline{0.0255}                       & 0.0100                     & 0.0028                       & 0.0038                          & 0.0022                       \\
\midrule
\rowcolor{beige}
 \begin{tabular}[c]{@{}c@{}}\textbf{TCN}\\ \textbf{w/ Concatenation} \end{tabular} &     \begin{tabular}[r]{@{}r@{}}\textbf{0.1309}\\ $\pm$ \textbf{0.037}\end{tabular}    &     \begin{tabular}[r]{@{}r@{}}\textbf{0.0521}\\ $\pm$ \textbf{0.005}\end{tabular}    & \begin{tabular}[r]{@{}r@{}}\textbf{0.0517}\\ $\pm$ \textbf{0.018}\end{tabular}    &     \begin{tabular}[r]{@{}r@{}}\textbf{0.0084}\\ $\pm$ \textbf{0.002}\end{tabular}    &       \begin{tabular}[r]{@{}r@{}}\textbf{0.0662}\\ $\pm$ \textbf{0.047}\end{tabular}    &          \begin{tabular}[r]{@{}r@{}}\textbf{0.0156}\\ $\pm$ \textbf{0.002}\end{tabular}        \\
\bottomrule
\end{tabular}}
\end{table}


\vspace{-2mm}

\subsection{Performance (Q2)}
\label{subsec:performance}

We evaluate the performance of \method combined with \costco and competitors in terms of AP@k since \method + CostCo achieves the highest performance among \method + TF methods.
Table~\ref{tab:performance} shows that \method + \costco outperforms competitors on all datasets except for AP@100 on SG data, achieving up to $288.9\%$ higher performance compared to existing methods.
Notably, \method significantly enhances performance on ML-100k and Gowalla datasets as \method generates enriched latent vectors by the effective aggregation of information from neighboring entities using a tensor-originated clique-expanded graph.
Traditional TF methods fail to learn effective latent vectors of entities due to high sparsity while HyperGCN cannot fuse latent vectors effectively.
Furthermore, the performance gap widens further when $k$ is large since \method helps to better predict interactions within the top-k interactions, particularly those at the lower ranks.
These results imply that \method generates enriched latent vectors even for entities with limited interactions.

\begin{table*}[!h]
\caption{
Comparison of the impact of \method on 3rd-order and 4th-order tensors in ML-100k.
\method is more effective on the 4th-order tensor than the 3rd-order tensor in terms of the average improvement.
}
\label{tab:impact_order}
\resizebox{0.99\textwidth}{!}{
\begin{tabular}{ll|rr|rr|rr|rr|rr|rr|r}
\toprule
\multicolumn{2}{c|}{Setting}                              & \multicolumn{2}{c|}{CP}                               & \multicolumn{2}{c|}{Tucker}                                   & \multicolumn{2}{c|}{NCFT} & \multicolumn{2}{c|}{M$^2$DMTF}   & \multicolumn{2}{c|}{NeAT}            & \multicolumn{2}{c|}{CostCo}                                   & \multicolumn{1}{c}{\multirow{2}{*}{\begin{tabular}[c]{@{}c@{}}Average\\ Improvement\end{tabular}}} \\
\multicolumn{1}{c}{Dataset} & \multicolumn{1}{c|}{Metric} & 
\multicolumn{1}{c}{w/o \method} & \multicolumn{1}{c|}{w/ \method} & 
\multicolumn{1}{c}{w/o \method} & \multicolumn{1}{c|}{w/ \method} & 
\multicolumn{1}{c}{w/o \method} & \multicolumn{1}{c|}{w/ \method} & 
\multicolumn{1}{c}{w/o \method} & \multicolumn{1}{c|}{w/ \method} & 
\multicolumn{1}{c}{w/o \method} & \multicolumn{1}{c|}{w/ \method} & 
\multicolumn{1}{c}{w/o \method} & \multicolumn{1}{c|}{w/ \method} & \multicolumn{1}{c}{}                                     \\
\midrule
\multirow{3}{*}{\begin{tabular}[c]{@{}c@{}}3rd-order tensor \\ (ML-100k)\end{tabular}}
& AP@100                      & 0.2548                 &       0.1646\blue{$\blacktriangledown$}                   & 0.2194                     & 0.2841\red{$\blacktriangle$}                           & 0.1972                   &         0.2523\red{$\blacktriangle$}              & 0.1976 & 0.2146\red{$\blacktriangle$} & 
0.1653 &  0.1695\red{$\blacktriangle$}  &
0.2749                     & \textbf{0.3105}\red{$\blacktriangle$}                          & 7.68\%                                                   \\
& AP@1000                     & 0.1746                 &       0.1308\blue{$\blacktriangledown$}                   & 0.1594                     & 0.1923\red{$\blacktriangle$}                           & 0.1688                   &        0.1956\red{$\blacktriangle$}                & 0.1389 & 0.1834\red{$\blacktriangle$} & 
 0.1262 &  0.1562\red{$\blacktriangle$}  &
0.1906                     & \textbf{0.2368}\red{$\blacktriangle$}                          & 15.24\%                                                  \\
& AP@10000                    & 0.0807                 &      0.0689\blue{$\blacktriangledown$}                 & 0.0766                     & 0.0928\red{$\blacktriangle$}                           & 0.0816                   &       0.0992\red{$\blacktriangle$}                  & 0.0650 & 0.0879\red{$\blacktriangle$} & 
0.0682 &  0.0869\red{$\blacktriangle$}  &
0.0956                     & \textbf{0.1179}\red{$\blacktriangle$}                          & 18.97\%     \\  
\midrule
\midrule
\multirow{3}{*}{\begin{tabular}[c]{@{}c@{}}4th-order tensor \\ (ML-100k)\end{tabular}} 
& AP@100                      & 0.0271                 &       0.0275\red{$\blacktriangle$}                   & 0.0566                     & 0.1309\red{$\blacktriangle$}                           & 0.0553                   &         0.1586\red{$\blacktriangle$}              & 0.0459 & 0.0943\red{$\blacktriangle$} & 
	0.0806 &  \textbf{0.2532} \red{$\blacktriangle$}    &
0.0845                     & 0.1309\red{$\blacktriangle$}                          & 115.6\%                                                   \\
& AP@1000                     & 0.0206                 &       0.0150\blue{$\blacktriangledown$}                   & 0.0355                     & 0.0705\red{$\blacktriangle$}                           & 0.0352                   &        0.1020\red{$\blacktriangle$}                & 0.0210 & 0.0568\red{$\blacktriangle$} &
	0.0401 & 0.0879 \red{$\blacktriangle$} &
 0.0509                     & \textbf{0.1048}\red{$\blacktriangle$}                          & 109.3\%                                                  \\
& AP@10000                    & 0.0101                 &      0.0071\blue{$\blacktriangledown$}                 & 0.0165                     & 0.0301\red{$\blacktriangle$}                           & 0.0162                   &       0.0392\red{$\blacktriangle$}                  & 0.0064 & 0.0218\red{$\blacktriangle$} &
	0.0176 & 0.0313 \red{$\blacktriangle$} &
 0.0255                     & \textbf{0.0521}\red{$\blacktriangle$}                          & 123.4\%     \\
\bottomrule
\end{tabular}}
\end{table*}

\begin{table}[!t]
\centering
\caption{Ablation study for GNN types on ML-100k data.
See detail in Section~\ref{subsec:ablation}.}
\label{tab:ablation}
\resizebox{0.9\columnwidth}{!}{%
\begin{tabular}{c|rrr}  
\toprule
{Method} & AP@100 & AP@1000 & AP@10000 \\
\midrule
CostCo (Baseline)             &        0.0845     &   0.0509     &   0.0255 \\
\midrule
GAT~\cite{velivckovic2017graph} with CostCo                 &        0.0957     &   0.0722     &   0.0412 \\
ChebConv~\cite{defferrard2016convolutional} with CostCo                 &        \textbf{0.1532}     &   \underline{0.1000}     &   \underline{0.0425} \\
APPNP~\cite{gasteiger2018predict} with CostCo             &        \underline{0.1355}     &   0.0757     &   0.0359 \\
FAConv~\cite{bo2021beyond} with CostCo             &        0.1011     &   0.0615     &   0.0324 \\
AntiSymmetric~\cite{gravina2022anti} with CostCo             &        0.1100     &   0.0551     &   0.0233 \\
\midrule
\rowcolor{beige}
\textbf{\method} with CostCo                   &        0.1309     &   \textbf{0.1048}     &   \textbf{0.0521} \\
\midrule
\end{tabular}
}
\resizebox{0.9\columnwidth}{!}{%
\begin{tabular}{c|rrr}  
\toprule
{Method} & AP@100 & AP@1000 & AP@10000 \\
\midrule
NCFT (Baseline)             &        0.0553     &   0.0352     &   0.0162 \\
\midrule
GAT~\cite{velivckovic2017graph} with NCFT               &        0.0337     &   0.0325     &   0.0203 \\
ChebConv~\cite{defferrard2016convolutional} with NCFT             &        \underline{0.0942}     &   \underline{0.0612}     &   \underline{0.0274} \\
APPNP~\cite{gasteiger2018predict} with NCFT             &        0.0808     &   0.0494     &   0.0247 \\
FAConv~\cite{bo2021beyond} with NCFT             &        0.0820     &   0.0660     &   0.0347 \\
AntiSymmetric~\cite{gravina2022anti} with NCFT             &        0.0957     &   0.0459     &   0.0198 \\
\midrule
\rowcolor{beige}
\textbf{\method} with NCFT                   &        \textbf{0.1586}     &   \textbf{0.1020}     &   \textbf{0.0392} \\
\bottomrule                                                                                        
\end{tabular}
}
\end{table}

\vspace{-2mm}
\subsection{Ablation Study (Q3)}
\label{subsec:ablation}
We evaluate the effectiveness of various operations used in \method:
(1) feature transformation and nonlinear activation function,
(2) mean, sum, product, and concatenation operations,
and (3) varied GNN types.

\textbf{Feature transformation matrix and nonlinear activation function.}
We evaluate the effectiveness of the following components that can be used in TCN:
(1) feature transformation (F) and (2) nonlinear activation (N).
We use the ReLU function for the nonlinear activation.
Table~\ref{tab:ablation_fn} shows that \method achieves high performance stably on all the datasets.
Although TCN+F+N achieves better performance on SG dataset, 
it has low AP@k on other datasets due to the learning instability.
The nonlinear and transformation functions do not have a significant impact on our problem since we need to utilize the latent vectors to be learned as input for a GNN, which poses training difficulties with nonlinear and transformation functions; in graph data, GNNs typically utilize the provided feature vectors as input during training.

\textbf{Mean, sum, element-wise product, and concatenation.}
We compare operations used in integrating \method with \costco (Section~\ref{subsec:tcn_based_tf}).
As shown in Table~\ref{tab:ablation_hmp}, \method with the concatenation outperforms \method with the other operations on all datasets except for UCI data.
The concatenation with \method outperforms other operations with \method since the concatenation provides input with preservation of the information of each matrix $\mat{F}^{[l]}$ for $l=0,..., L$, allowing a TF predictor to capture complex patterns.

\textbf{Varied GNN types.}
We evaluate the performance by varying GNN types on ML-100k.
We utilize CostCo and NCFT as TF methods, and \method (i.e., a linear propagation), GAT~\cite{velivckovic2017graph}, ChebConv~\cite{defferrard2016convolutional}, APPNP~\cite{gasteiger2018predict}, FAConv~\cite{bo2021beyond}, and AntiSymmetric~\cite{gravina2022anti} as message passing mechanisms.
Table~\ref{tab:ablation} shows that the integration of message-passing mechanisms achieves consistent performance enhancements for CostCo and NCFT with a few exceptions: (1) GAT with NCFT for AP@100 and AP@1000, and (2) AntiSymmetric with CostCo for AP@10000.
Note that we found limited performance gains with GAT due to training instability. 
Furthermore, \method achieves better performance than the other GNNs since \method is straightforward to optimize rather than them.
These results imply that (1) our tensor-originated graph helps to improve the performance of traditional TF methods regardless of GNN types, and (2) \method is very competitive among GNNs. 


\vspace{-2mm}
\subsection{Impact of \method on 3rd-order and 4th-order Tensors (Q4)}
\label{subsec:impact_order}
We evaluate the impact of \method by varying the order of tensors in ML-100k: 3rd-order tensor and 4th-order tensor.
We use the 4th-order tensor in Table~\ref{tab:Description} and construct a 3rd-order tensor where its form is (user, movie, month), its size is $943\times 1663 \times 8$, and the number of observations is $999,811$.
Compared to the 4th-order tensor, this 3rd-order tensor has some missing observations and items during the preprocessing stage. However, this does not significantly affect the comparison between the two tensors.

Table~\ref{tab:impact_order} shows that \method has a greater impact on the 4th-order tensor than the 3rd-order tensor in terms of average improvement.
This performance gap is because \method works effectively on highly sparse tensors as discussed in Sections~\ref{subsec:compatibility} and~\ref{subsec:performance}.
For the 4th-order tensor, which is sparser than the 3rd-order tensor, existing tensor factorization methods struggle to learn effective latent vectors of entities that have only a few interactions.
\method universally improves the performance of existing tensor factorization methods on the 4th-order tensor by generating expressive latent vectors of entities.

\begin{table}[!t]
\caption{Comparison between \costco and \method+\costco in terms of training time and inference time.
}
\label{tab:running_time}
\resizebox{0.49\textwidth}{!}{
\begin{tabular}{l|rr|rr|rr}
\toprule
\multicolumn{1}{c|}{}                              & \multicolumn{2}{c|}{ML-100k}                               & \multicolumn{2}{c|}{SG}       & \multicolumn{2}{c}{Gowalla}                                                            \\
 \multicolumn{1}{c|}{Time} & \multicolumn{1}{c}{\costco} & \multicolumn{1}{c|}{\textbf{TCN}+\costco} & \multicolumn{1}{c}{\costco} & \multicolumn{1}{c|}{\textbf{TCN}+\costco} & \multicolumn{1}{c}{\costco} & \multicolumn{1}{c}{\textbf{TCN}+\costco}                            \\
\midrule
 Training time (sec)                      & 10.82                 &     12.307                & 9.734                     & 10.63   & 25.937                     & 98.841                                                                                         \\
Inference time (sec)                     & 5.377                 &           11.030                   & 2.705                     & 4.027    & 428.41                     & 869.39 \\
\bottomrule
\end{tabular}}
\end{table}

\vspace{-2mm}
\subsection{Computational time of \method (Q5)}
\label{subsec:time}
We analyze how much \method causes additional overhead when combined with \costco.
We measure training time per epoch while measuring inference time for computing scores of all unobserved interactions on a workstation with Tesla V100.
Table~\ref{tab:running_time} shows that \method performs well in GPU-based environments and does not impose a significant overhead in terms of training and inference times.
Note that the overhead by \method includes $L$ matrix multiplications $\mat{A}\mat{F}^{[l]}$ and the increase of the size of a TF predictor due to the concatenation.
\method + CostCo requires more training and inference time than CostCo on Gowalla dataset, but not to the extent that it leads to learning and inference failure.
In addition, \method + CostCo demonstrates the performance enhancement in terms of AP@k.

\vspace{-2mm}
\subsection{Computational Time and Space Requirements for Graphs (Q6)}
\label{subsec:time_and_space}

We report running time and space cost for constructing a hypergraph and clique-expanded graph from a sparse tensor.
Table~\ref{tab:run_and_space} shows that we obtain the hypergraphs and clique-expanded graphs at reasonable time and space.
As described in Section~\ref{subsec:hypergraph}, hypergraph construction and clique-expanded graph construction have complexities of $O(NE)$ and $O(N^2E)$, respectively, but the experimental difference in space cost is not as large as the theoretical difference. 
This is because clique-expanded graphs can have fewer non-zeros, depending on the sparse patterns of the data, requiring less space cost.
Compared to training and inference time in Section~\ref{subsec:time}, the construction time does not impose a significant overhead on the overall process.

\section{Related Works}
\label{sec:related}

In this section, we review related works on tensor factorization, hyperedge prediction methods, and tensor-based applications.

\textbf{Tensor factorization methods.}
Many works have been developed tensor factorization methods for analyzing higher-order tensors.
CP (CANDECOMP/PARAFAC) decomposition~\cite{carroll1970analysis,harshman1970foundations,kiers2000towards}
factorizes a given tensor into factor matrices which construct rank-one tensors as described in~\cite{kolda2009tensor}.
Tucker decomposition~\cite{tucker1966some} factorizes a given tensor into factor matrices and core tensor.
CostCo~\cite{liu2019costco} is a CNN-based tensor factorization model that captures nonlinear interactions between factors.
NTF~\cite{wu2019neural} is an RNN-based tensor factorization model that captures sequential patterns in a real-world tensor.
M$^2$DMTF~\cite{fan2021multi} employs nonlinear transformation for its factor matrices.
We can view M$^2$DMTF as consisting of an MLP encoder and predictor that is the same as the predictor of Tucker decomposition; we experimentally show the compatibility of \method with the MLP encoder of M$^2$DMTF.
NeAT~\cite{ahn2024neat} is a non-linear tensor factorization method with interpretability.
In our experiments, we use only tensor factorization methods that can be easily adapted to our problem;
we do not compare \method with NTF since NTF works only on a temporal tensor which has the time dimension.

\textbf{Hypergraph methods.}
Previous works~\cite{choo2022persistence, gao2020hypergraph, cencetti2021temporal, schawe2022higher} mainly deal with hypergraphs. Note that none of them~\cite{choo2022persistence, cencetti2021temporal, schawe2022higher} handles a tensor structure. Unlike a prior study~\cite{gao2020hypergraph}, which used a tensor structure to handle hypergraphs dynamically, we employ a hypergraph structure to effectively handle sparse tensors.
Many previous works~\cite{zhang2019hyper,yoon2020much,yadati2020nhp} have focused on predicting hyperedges when a hypergraph is given.
Furthermore, Hwang et al.~\cite{hwang2022ahp} and Ko et al.~\cite{ko2023enhancing} enhance hypergraph edge prediction models by utilizing adversarial training and contrastive learning techniques, respectively.
Our work bridges the gap between tensor and hypergraph research, complementary to the aforementioned works.
However, there are still differences in various aspects:
Tensor factorization methods utilize learnable latent vectors as an input while hyperedge prediction methods utilize input node features.
All hyperedges in a tensor-based hypergraph are the same with the size $N$, while real-world hypergraphs have varying sizes of hyperedges.
For a tensor-based hypergraph, all nodes belong to one of the types corresponding to dimensions (e.g., user, item, time).
Therefore, we do not provide an exhaustive comparison with hyperedge prediction methods in this paper.

\begin{table}[!t]
\caption{Running time and space cost for constructing a hypergraph and clique-expanded graph. We measure the time and space cost to construct the CSR (Compressed Sparse Row) format for storing matrices of a hypergraph and clique-expanded graph.
(H) and (C) indicate a hypergraph construction and clique-expanded graph construction, respectively.
}
\label{tab:run_and_space}
\resizebox{0.999\columnwidth}{!}{%
\begin{tabular}{crrrr}	
\toprule
\multicolumn{1}{c}{Data} & \multicolumn{1}{c}{Running time (H)} & \multicolumn{1}{c}{Running time (C)} & \multicolumn{1}{c}{Space (H)} & \multicolumn{1}{c}{Space (C)} \\
\midrule
ML-100k              &        0.01 (sec)                                                                                              &   0.04 (sec)                                                                                                &   4.02 MB                                                                                            
	&   2.74 MB\\	
SG                   &        0.02 (sec)                                                                                              &   0.03 (sec)                                                                                                &   2.57 MB                                                                                            
	&   2.87 MB                                                                                         \\
Gowalla              &        0.04 (sec)                                                                                              &   0.51 (sec)                                                                                                &   20.33 MB                                                                                            
	&   25.48 MB\\
UCI                  &        0.006 (sec)                                                                                              &   0.01 (sec)                                                                                                &   0.56 MB                                                                                            
	&   0.50 MB                                                                                         \\
UMLS                 &        0.002 (sec)                                                                                              &   0.01 (sec)                                                                                                &   0.16 MB                                                                                            
	&   0.11 MB                                                                                         \\
NYC                 &        0.03 (sec)                                                                                              &   0.11 (sec)                                                                                                &   5.60 MB                                                                                            
	&   8.34 MB                                                                                         \\
Yelp                 &        0.04 (sec)                                                                                              &   0.15 (sec)                                                                                                &   8.34 MB                                                                                            
	&   11.56 MB                                                                                         \\	
Yahoo                 &        0.08 (sec)                                                                                              &   0.38 (sec)                                                                                                &   19.51 MB                                                                                            
	&   28.15 MB                                                                                         \\	    
\bottomrule                                                                                        
\end{tabular}
}
\end{table}

\textbf{Tensor-based applications.} 
Many previous works~\cite{lacroix2018canonical,BalazevicAH19,liu2020generalizing} developed tensor decomposition-based methods for knowledge base completion.
Furthermore, Lacroix et al.~\cite{BalazevicAH19} developed a TF-based method for temporal knowledge completion.
Tensor factorization has been widely used for higher-order recommendation~\cite{rendle2010pairwise,fang2015personalized,choi2019s3,chen2020neural}.
Many works~\cite{tomasi2005parafac,acar2011scalable,smith2016exploration,liu2014generalized} have focused on developing tensor completion methods that predict specific values of unobserved entries accurately.
In this paper, we do not compare \method with tensor completion methods since their main concern is to accurately predict missing {\em values} with the mean squared loss (MSE) function rather than identifying the {\em existence} of a given interaction.
Still, our work is complementary to tensor completion methods since tensor completion methods, similarly to tensor factorization methods used in this paper, can also be well-compatible with \method.
In addition to tensor completion, there are a few works that predict future values for a given tensor.
Net$^{3}$~\cite{jing2021network} is developed for missing and future values when a tensor time series and graphs corresponding to each dimension are given.
In contrast to our paper, Net$^{3}$ utilizes GCN (Graph Convolutional Network) for obtaining powerful representation vectors from the given graphs.
TensorCast~\cite{de2017tensorcast} is a coupled tensor factorization model to predict future values accurately and efficiently when a sparse tensor (e.g., Tweeter data of the form (user-hashtag-time)) and the additional information (e.g., Tweeter data of the form (user-user-time)) are given.

\vspace{-2mm}
\section{Discussion and Future Work}
\label{sec:discussion}
We discuss a scalability issue of predicting scores of all unobserved interactions, which suffers from the exponential growth in the number of interactions as the number of entities increases.
For NYC, Yelp, and Yahoo datasets, we evaluate the performance using sampled interactions instead of all unobserved interactions due to a large number of unobserved interactions.
It is noteworthy that this issue, inherent to the top-$k$ prediction problem, is not unique to \method but is commonly encountered by existing tensor factorization methods as well.
Presumably, this issue was a major factor in preventing existing TF methods from deeply dealing with the top-$k$ interaction problem. 
Although we show that \method achieves higher performance on large-scale tensors, such as SG, Gowalla, NYC, Yelp, and Yahoo datasets, which have billion-scale interactions, overcoming this challenge remains a crucial task for future research.
Therefore, our future directions include the improvement of scalability to deal with an order of magnitude larger (i.e., trillion-scale) sparse tensors in the top-$k$ interaction problem.
There is a potential way that addresses the scalability issue:
to develop an efficient approach that finds the top-$k$ interactions without fully computing the scores of all unobserved interactions, and without constructing a sampled set.

\section{Conclusion}
\label{sec:conclusion}

In this paper, we propose \method, an accurate tensor convolutional network that is fully compatible with a variety of tensor factorization methods for top-{\em k} interaction prediction.
Given a sparse tensor, we first find a hypergraph and its clique-expanded graph originating from the tensor.
Then, \method extracts relation-aware latent vectors over the graph and uses them as an input of a tensor factorization predictor.
For finding top-$k$ higher-order interactions in a sparse tensor, we experimentally show that \method with a tensor factorization method consistently achieves the best performance and boosts the performance of various tensor factorization methods.



\bibliographystyle{ACM-Reference-Format}
\balance
\bibliography{mybib}

%

\end{document}